\documentclass{article}

\usepackage{graphicx,amsmath,amsfonts,amssymb,bm,url,epsfig,epsf,color,fullpage,mathbbol,fmtcount,algorithmic,algorithm,semtrans,caption,subcaption,multirow,amsthm}

\usepackage{enumitem}
\usepackage[title]{appendix}
\usepackage[utf8]{inputenc} 
\usepackage[T1]{fontenc}    
\usepackage{url}            
\usepackage{booktabs}       
\usepackage{amsfonts}       
\usepackage{nicefrac}       
\usepackage{microtype}      
\usepackage{hhline}

\usepackage[T1]{fontenc}
\usepackage[utf8]{inputenc}
\usepackage{graphicx}
\usepackage{subcaption}
\usepackage{color}
\usepackage{amsmath}
\usepackage{bbm}
\usepackage{tcolorbox}
\usepackage{lipsum}  
\usepackage{tcolorbox}
\usepackage{amsfonts,amssymb}
\usepackage{mathtools}
\usepackage{commath}
\usepackage{relsize}
\usepackage{comment,color,soul}
\usepackage{amsfonts}
\usepackage{url}
\usepackage{lipsum}
\usepackage{nicefrac}
\usepackage{float}
\usepackage{titling}

\usepackage[hidelinks,backref=page]{hyperref}
\definecolor{darkred}{RGB}{150,0,0}
\definecolor{darkgreen}{RGB}{0,150,0}
\definecolor{darkblue}{RGB}{0,0,150}
\hypersetup{colorlinks=true, linkcolor=darkred, citecolor=darkblue, urlcolor=darkblue}

\date{}

\author{\\ Farzan Farnia\thanks{Laboratory for Information \& Decision Systems, Massachusetts Institute of Technology, email: farnia@mit.edu}, Amirali Aghazadeh\thanks{Department of Electrical Engineering and Computer Science, University of California, Berkeley, email: amiralia@berkeley.edu}, James Zou\thanks{Department of Biomedical Data Science, Stanford University, email: jzou@stanford.edu}, David Tse\thanks{Department of Electrical Engineering, Stanford University, email: dntse@stanford.edu}}

\usepackage[round]{natbib}

\newtheorem{thm}{Theorem}
\newtheorem*{thm*}{Theorem}

\newtheorem*{remark*}{Remark}
\newtheorem{lemma}{Lemma}

\newtheorem{cor*}{Corollary*}

\DeclareMathAlphabet{\mathsfit}{\encodingdefault}{\sfdefault}{m}{sl}
\SetMathAlphabet{\mathsfit}{bold}{\encodingdefault}{\sfdefault}{bx}{n}

\allowdisplaybreaks

\usepackage{algorithm}
\usepackage{algorithmic}  

\title{Group-Structured Adversarial Training}

\begin{document}

\maketitle

\begin{abstract}
Robust training methods against perturbations to the input data have received great attention in the machine learning literature. A standard approach in this direction is adversarial training which learns a model using adversarially-perturbed training samples. However, adversarial training performs suboptimally against perturbations structured across samples such as universal and group-sparse shifts that are commonly present in biological data such as gene expression levels of different tissues. In this work, we seek to close this optimality gap and introduce Group-Structured Adversarial Training (GSAT) which learns a model robust to perturbations structured across samples. We formulate GSAT as a non-convex concave minimax optimization problem which minimizes a group-structured optimal transport cost. Specifically, we focus on the applications of GSAT for group-sparse and rank-constrained perturbations modeled using group and nuclear norm penalties. In order to solve GSAT's non-smooth optimization problem in those cases, we propose a new minimax optimization algorithm called GDADMM by combining Gradient Descent Ascent (GDA) and Alternating Direction Method of Multipliers (ADMM). We present several applications of the GSAT framework to gain robustness against structured perturbations for image recognition and computational biology datasets.
\end{abstract}

\section{Introduction}
Robust learning schemes are the key to reliably deploy statistical learning models in high-risk applications such as self-driving cars and healthcare systems. In the machine learning literature, several frameworks have been proposed to improve robustness in various learning applications. All these frameworks attempt to ensure that the trained model remains robust under certain changes to the data distribution. A popular approach in this direction is to optimize a model's worst-case performance against an ambiguity set around the input distribution. A critical step toward developing such worst-case methods is to find a proper model for the potential uncertainties in the collected data. 

Due to the significance of modern deep learning applications, robust training for deep neural net (DNN) models has recently received enormous attention. While DNNs have achieved super-human scores over several benchmark datasets, they have been observed to lack robustness to minor adversarially-designed perturbations added to the input data widely-known as \emph{adversarial attacks} \citep{szegedy2013intriguing,biggio2013evasion}. Adversarial training \citep{goodfellow2014explaining,madry2017towards}, which is to train a model using adversarially-perturbed training examples, is a standard approach to train DNN classifiers robust against adversarial attacks. \citet{madry2017towards} show that adversarial training can be interpreted as finding the model with the optimal worst-case performance against norm-bounded perturbations independently generated across samples. While the underlying uncertainties in many learning applications can be potentially different from independent norm-bounded perturbations, \citep{madry2017towards}'s original scheme against norm-bounded perturbations is still widely-considered as the standard robust training method for DNNs.

In this work, we focus on the applications of adversarial training to learning tasks where the perturbations to different data points are \emph{non-independently distributed across samples}. Here, the additive perturbations are assumed to be further constrained to certain structures across both training and test samples, resulting in correlated perturbations to different data points. This assumption holds in several applications of interest in machine learning. For example, \emph{universal adversarial perturbations} \citep{moosavi2017universal} perturb every sample in an image recognition task with an identical perturbation which leads to completely correlated perturbations across samples. Standard adversarial training methods are therefore unable to fully capture the dependencies of universal perturbations.

Structured perturbations are not limited to image datasets and are also present in many computational biology applications during test time. In biological data collections, the samples gathered in different experimental batches typically show significant variations recognized as the \emph{batch effects}. Batch effects are modeled as the addition of a single perturbation vector to the data points in every batch \citep{luo2010comparison}, implying that collectively a universal set of perturbations affect training and test data. Therefore, in order to train a model that can generalize from one set of batches to another set, the learner needs robustness against \emph{universal sets of perturbations}.

As another example from computational biology, cell types are known to cause significant variations in the gene expression levels of a small subset of genes \citep{rahmani2016sparse}. As a result, a sparse set of gene expression variables will undergo distribution shifts if cell types differ between training and test samples. Therefore, for learning a classification rule that generalizes from one cell type to another, the learner needs to be robust to group-sparse perturbations. Such robustness against structured group-sparse perturbations will lead to generalization from training samples with one cell type to test samples with other cell types.

In all the discussed applications, proper generalization from training to test data requires robustness against certain \emph{group-structured perturbations}. A natural question is whether standard adversarial training can be further improved in learning under such group-structured perturbations. In particular, \citet{shafahi2018universal} demonstrate that in the special case of universal adversarial perturbations the improvement on test data is achievable by constraining adversarial training to an identical perturbation across training examples. Can we more generally extend adversarial training to gain robustness against group-structured perturbations? In this work, we address this question by proposing \emph{Group-Structured Adversarial Training (GSAT)} as a robust learning framework against group-structured perturbations. 

\section{Contributions}
This work focuses on developing robust learning algorithms against various types of group-structured perturbations.
To develop the GSAT framework, we first provide a generalization of optimal transport problems for statistically modeling distribution shifts under structured perturbations. Based on this generalization, we find the optimal transportation map minimizing a group-structured transportation cost for transporting a group of samples between two probability domains. We generalize the primary Kantorovich duality theorem \citep{villani2008optimal} from optimal transport theory to group-structured settings for analyzing GSAT under permutation-invariant transportation costs. 
 
 We demonstrate that every group-structured transportation cost targets a specific type of group-structured perturbations including identical, group-sparse, and low-rank perturbations. Next, we reduce GSAT's robust learning problem to a non-convex concave minimax optimization problem. In order to handle the non-smooth nature of the formulated minimax problem, we propose GDADMM as a minimax optimization algorithm combining the well-known Gradient Descent Ascent (GDA) and Alternating Direction Method of Multipliers (ADMM) \citep{boyd2011distributed}. We prove that GDADMM is guaranteed to converge to a stationary minimax point of GSAT's optimization problem. 
 
 Furthermore, we explore GSAT's application for robust feature selection by selecting the variables whose structured perturbation affects the classification accuracy the most. We present the numerical results of applying GSAT to  different image recognition and computational biology problems. Our contributions can be summarized as follows:
\begin{itemize}[leftmargin=*,itemsep=0.3mm]
    \item {\bf Theory:} We provide a generalization of the optimal transport problem to group-structured settings.
    \item {\bf Methodology:} We develop GSAT as a robust learning framework against group-structured perturbations.
    \item {\bf Algorithm:} We propose GDADMM as a minimax optimization algorithm for non-smooth minimax problems with convergence guarantees.
    \item {\bf Experiment:} We show the application of the GSAT framework in image recognition and computational biology.
\end{itemize}

\section{Related work}
Improving the robustness of deep learning algorithms has been extensively studied in the literature. A large body of related works \citep{goodfellow2014explaining,kurakin2016adversarial,moosavi2016deepfool,carlini2016defensive,carlini2017towards,tramer2018ensemble,tramer2020adaptive} develop various gradient-based defense methods to train robust classifiers against adversarial attacks. Moreover, developing defense schemes with certifiable robustness against adversarial attacks has been the subject of several related works \citep{sinha2017certifiable,raghunathan2018certified,wong2018provable,levine2020robustness,singla2020second,chiang2020certified,ghiasi2020breaking}. Also, Wasserstein adversarial attacks with standard Wasserstein distances have been studied in multiple related references \citep{wong2019wasserstein,levine2020wasserstein,hu2020improved}. 

Several related works have focused on adversarial perturbations that are structured \emph{across features}. References \citep{wang2014sparse,zhang2016adversarial,xiao2015feature,sharma2017attacking,chen2017ead,marzi2018sparsity,shafahi2018are,chen2018ead,xu2018structured,fan2020sparse} study sparse adversarial attack and defense schemes,
where an $L_0$ or $L_1$-norm function is optimized to impose sparse perturbations across features. The related work \citep{laidlaw2019functional} studies adversarial perturbations to certain functions of the input data. However, unlike our proposed framework the designed perturbations in those schemes are structured across features and not across samples. 
As a special case of perturbations structured across samples, achieving robustness against universal adversarial perturbations \citep{moosavi2017universal} has been studied in the related references \citep{shafahi2018are,moosavi2017analysis,akhtar2018defense}. Our work generalizes the defense schemes against universal perturbations to other types of group-structured perturbations. 

We note that the application of optimal transport costs in robust learning frameworks has been studied in multiple related works \citep{abadeh2015distributionally,esfahani2018data, lee2018minimax,shafieezadeh2019regularization}. These works develop minimax learning frameworks in which the loss function is optimized in a Wasserstein ambiguity set around the data distribution. However, their proposed approaches are based on standard optimal transport costs which cannot capture the potential structures across samples. 


\section{Preliminaries}
\subsection{Notation}
Throughout this paper, we use $\mathbf{X}$ and $Y$ to denote the $d$-dimensional feature vector and the label variable for a labeled sample $(\mathbf{X},Y)$, respectively. We also use $\boldsymbol{\delta}$ to denote the perturbation vector added to the feature vector. $\underline{\mathbf{X}}^m := [\mathbf{X}_1;\ldots ; \mathbf{X}_m]$ stands for the $m\times d$ feature matrix of a group of $m$ samples with its $i$th row containing the feature vector of the $i$th sample $\mathbf{X}_i$. Similarly, $\underline{Y}^m := [Y_1;\ldots ; Y_m]$ denotes the vector including the label variables of the $m$ group samples and $\underline{\boldsymbol{\delta}}^m:=[\boldsymbol{\delta}_1;\ldots ; \boldsymbol{\delta}_m]$ denotes the $m\times d$ perturbation matrix including the perturbations of the $m$ group samples.

Regarding the matrix norms, we use $\Vert \cdot\Vert_F$ and $\Vert \cdot \Vert_{*}$ to respectively denote the Frobenius norm and the nuclear norm, i.e. the sum of the singular values. We also denote the $\ell_{1,2}$-group norm function as $\Vert A_{m\times d} \Vert_{1,2}:= \sum_{i=1}^d \Vert A_{:,i}\Vert_2$ which is the summation of the Euclidean norm of  $A$'s columns.

\subsection{Supervised Learning  \& Adversarial Perturbations}
Given $n$ training samples $(\mathbf{x}_i,y_i)_{i=1}^n$ and loss function $\ell$, the goal in supervised learning is to find the optimal prediction rule in a parametric function space $ \mathcal{F}=\{f_{\mathbf{w}}: \mathbf{w}\in\mathcal{W} \}$ minimizing the expected loss (risk) $\mathbb{E}\bigl[\ell(f_{\mathbf{w}}(\mathbf{X}),Y)\bigr]$. Here the risk is evaluated over the test samples drawn from the underlying distribution $P_{\mathbf{X},Y}$. To do this, a standard learning approach called \emph{empirical risk minimization (ERM)} minimizes the empirical risk averaged over the training samples:
\begin{equation}
    \min_{\mathbf{w}\in\mathcal{W}}\: \frac{1}{n}\sum_{i=1}^n \ell\bigl( f_{\mathbf{w}}(\mathbf{x}_i),y_i \bigr).
\end{equation}
However, the ERM learner is shown to lack robustness against adversarial perturbations to its input. A standard approach to generate adversarial perturbations is by maximizing the loss function over a norm-ball around a data point $(\mathbf{x}_i,y_i)$:
\begin{equation}
    \underset{\boldsymbol{\delta}_i:\, \Vert \boldsymbol{\delta}_i \Vert\le \epsilon}{\arg\!\max}\: \ell\bigl( f_{\mathbf{w}}(\mathbf{x}_i+\boldsymbol{\delta}_i),y_i \bigr).
\end{equation}
Here the adversarial perturbation is generated independently for every training sample. In contrast, universal adversarial perturbations introduced in \citep{moosavi2017universal} use an identical perturbation for all samples. A universal perturbation can be generated by solving the following optimization problem:
\begin{equation}
    \underset{\boldsymbol{\delta}:\, \Vert \boldsymbol{\delta} \Vert\le \epsilon}{\arg\!\max}\: \frac{1}{n}\sum_{i=1}^n \ell\bigl( f_{\mathbf{w}}(\mathbf{x}_i+\boldsymbol{\delta}),y_i \bigr).
\end{equation}

\subsection{Optimal Transport Costs \& Distributionally Robust Adversarial Training}

In the literature, the tools from optimal transport theory have been applied to develop several robust learning frameworks 
\citep{sinha2017certifiable,esfahani2018data,lee2018minimax,shafieezadeh2019regularization}.
To review the related concepts of optimal transport theory, the optimal transport cost $W_c(P,Q)$ for cost function $c(\mathbf{z},\mathbf{z}')$ between data points $\mathbf{z}=(\mathbf{x},y),\, \mathbf{z}'=(\mathbf{x}',y')$ is defined as
\begin{equation}\label{Definition, optimal transport cost}
  W_c(P,Q) := \inf_{M\in\Pi(P,Q)}\: \mathbb{E}_M\bigl[c(\mathbf{Z},\mathbf{Z}') \bigr] .
\end{equation}
Here $\Pi(P,Q)$ denotes the set of all joint distributions on pair $(\mathbf{Z},\mathbf{Z}')$ marginally distributed according to $P_{\mathbf{X},Y}$ and $Q_{\mathbf{X},Y}$. The Kantorovich duality theorem \citep{villani2008optimal} shows that
\begin{equation}
    W_c(P,Q) = \max_{D}\: \mathbb{E}_P[D(\mathbf{Z})] - \mathbb{E}_Q[D^c (\mathbf{Z}')],
\end{equation}
where the $c$-transform of function $D$ is defined as $D^c(\mathbf{z}'):=\max_\mathbf{z}\, D(\mathbf{z}) - c(\mathbf{z},\mathbf{z}')$.

To learn a robust prediction rule, we consider a minimax learner with a Lagrangian penalty on the perturbed distribution's optimal transport cost to the empirical distribution of data $\hat{P}_n$:
\begin{equation}\label{Distributionally robust optimization}
    \min_{\mathbf{w}\in\mathcal{W}}\: \max_Q\: \mathbb{E}_Q\bigl[\, \ell(f_{\mathbf{w}}(\mathbf{X}),Y)\, \bigr] - \lambda W_c(Q,\hat{P}_n).
\end{equation}
Applying the Kantorovich duality, the above minimax problem can be reduced to a risk minimization problem. 
We define $\ell\circ f^{\lambda c}$ as the c-transform of the composition of $\ell$ and $f$: $\ell\circ f^{\lambda c} (\mathbf{x},y) := \max_{\mathbf{x}',y'}\, \ell(f(\mathbf{x}'),y') - \lambda c\bigl((\mathbf{x}',y'),(\mathbf{x},y)\bigr)$. Then, the minimax problem reduces to an ERM problem with the loss evaluated at the transported data points:
\begin{equation}\label{Distributionally robust ERM}
\min_{\mathbf{w}\in\mathcal{W}}\: \frac{1}{n}\sum_{i=1}^n \ell\circ f_\mathbf{w}^c (\mathbf{x}_i,y_i).
\end{equation}

\section{Optimal Transport Theory Generalized to Group-structured Settings}\label{Section: Group Optimal Transport}
Optimal transport theory provides the essential tools for developing a distributionally robust learning framework against adversarial perturbations. Here, the adversary transports the training samples from the data distribution to the perturbed samples in the adversarial domain. However, the perturbations are still generated independently across samples if ordinary optimal transport costs are applied.

To develop a distributionally robust framework against group-structured perturbations, we consider a \emph{group-structured optimal transport problem} where the goal is to optimally transport a group of samples between  two distributions given a specific group cost function. Consider two joint distributions $P,\, Q$ on $(\mathbf{X},Y)$ and a normalized transportation cost $\frac{1}{m}c_m(\underline{\mathbf{x}}^m,{\underline{\mathbf{x}}'}^m)$ representing the normalized cost for transporting $\underline{\mathbf{x}}^m := [\mathbf{x}_1;\ldots ; \mathbf{x}_m]$ to $\underline{\mathbf{x}}'^m := [\mathbf{x}'_1;\ldots ; \mathbf{x}'_m]$. We begin by defining \emph{group optimal transport costs} denoted by $W_{c_m}$:
\begin{equation}\label{Group optimal transport definition}
    W_{c_m} ( P, Q) := \inf_{\tiny \Pi( \underbrace{P\times\ldots \times P}_{\tiny m\, \text{\rm times}},\, \underbrace{ Q,\ldots,Q}_{\tiny m\, \text{\rm times}}) } \: \mathbb{E} \bigl[ \frac{1}{m}c_m(\underline{\mathbf{X}}^m,{\underline{\mathbf{X}}'}^m)  \bigr].
\end{equation}
In this definition, $ \Pi( \underbrace{P\times\ldots \times P}_{\tiny m\, \text{\rm times}},\, \underbrace{ Q,\ldots,Q}_{\tiny m\, \text{\rm times}})$ is the set of all couplings between the joint distribution of $m$ i.i.d. samples from $P$, denoted by $P^m = \underbrace{P\times\ldots \times P}_{\tiny m\, \text{\rm times}}$, and an $m$-dimensional joint distribution with all its first-order marginals fixed to be $Q$. Here, we transport a group of samples $\underline{\mathbf{X}}^m$ independently drawn from $P$ to another group $\underline{\mathbf{X}}'^m$ with the marginal distribution of each $\mathbf{X}'_i$ being $Q$. Note that the definition does not require an independent structure for the transported samples $\underline{\mathbf{X}}'^m$, since the added group-structured perturbations are in general correlated.

The following theorem generalizes the Kantorovich duality theorem to permutation-invariant group optimal transport costs. The theorem assumes that the group cost function is permutation invariant in order to reduce the complexity of solving \eqref{Group optimal transport definition}'s dual optimization problem. This assumption usually holds in robust learning applications, since altering the order of samples and perturbations is not supposed to alter the group transportation cost.     
\begin{thm}
Suppose $c_m$ is a non-negative lower-semi continuous group cost function. Assume that $c_m$ is permutation invariant, i.e. for every permutation $\pi$ we have $c_m\bigl(\underline{\mathbf{x}}^m,{\underline{\mathbf{x}}'}^m\bigr) = c_m\bigl(\pi(\underline{\mathbf{x}}^m),\pi({\underline{\mathbf{x}}'}^m)\bigr)$. Then,
\begin{equation}
W_{c_m} ( P, Q) = \max_{D}\: \mathbb{E}_Q[D(\mathbf{X})] - \mathbb{E}_{P^m}[D^{c_m} (\underline{\mathbf{x}}'^m)],   
\end{equation}
where we define a function's $c_m$-transform as $D^{c_m}(\underline{\mathbf{x}}'^m):= \max_{\underline{\mathbf{x}}^m} \frac{1}{m}\bigl[ \sum_{i=1}^m D(\mathbf{x}_i) - c_m(\underline{\mathbf{x}}^m,\underline{\mathbf{x}}'^m)\bigr]$.
\end{thm}
\begin{proof}
We defer the proof to the Appendix. 
\end{proof}
Next, we discuss three examples of group cost functions resulting in universal, group-sparse, low-rank perturbations.
 
\subsection{Indicator Cost: Universal Perturbations}

To address universal perturbations, we need to consider a group cost imposing an identical perturbation for every input sample. We propose applying an indicator cost function forcing every perturbation ${\boldsymbol{\delta}}_i := \mathbf{x}_i - \mathbf{x}'_i$ to be identical to the average perturbation $\overline{\boldsymbol{\delta}}=\frac{1}{m}\sum_{i=1}^m {\boldsymbol{\delta}}_i$. This choice leads to the following group cost function
\begin{equation}\label{Group cost: Indicator}
    c^{\text{\rm Ind}}\bigl(\underline{\mathbf{x}}^m,\underline{\mathbf{x}}^m+\underline{\boldsymbol{\delta}}^m\bigr) := \alpha \sum_{i=1}^n \bigl[ \mathbf{1}\bigl( \boldsymbol{\delta}_i \neq \overline{\boldsymbol{\delta}} \bigr)\bigr] +  (1-\alpha) \Vert \underline{\boldsymbol{\delta}}^m \Vert^2_F.
\end{equation}
Here $\alpha\in [0,1]$ is a fixed constant, and $\underline{\boldsymbol{\delta}}^m = [\boldsymbol{\delta}_1;\ldots; \boldsymbol{\delta}_m]$ is the perturbation matrix. Also, the indicator function is defined as
\begin{equation*}
\mathbf{1}(\mathbf{x}\neq \mathbf{x}'):=\begin{cases} \begin{aligned}0\quad &\text{\rm if}\;\; \mathbf{x} = \mathbf{x}', \\
\infty\quad &\text{\rm otherwise.}
\end{aligned}
\end{cases}
\end{equation*}
Considering $0\le\alpha\le 1$, $\alpha=1$ requires an identical perturbation without putting any constraints on the perturbation's magnitude. On the other hand, $\alpha=0$ leads to the standard optimal transport problem since the Frobenius norm term reduces to the sum of individual $\Vert\boldsymbol{\delta}_i\Vert^2$'s.  

\subsection{Group Norm Cost: Group-sparse Perturbations}
We introduce a cost function imposing group-sparse perturbations with a common sparsity pattern across samples. Group-sparse perturbations can model cell-type effects for genomics datasets, because different cell-types express different gene expression levels across a sparse subset of genes \citep{rahmani2016sparse}. Group-sparse perturbations can be further applied for robust feature selection, identifying the most relevant features whose perturbation influences the model's performance the most. 

In the literature, group-norm functions \citep{yuan2006model} are widely-used to learn shared sparsity patterns across features. Motivated by this success, we use group norms and for constant $0\le \alpha \le 1$ define the following group cost function
\begin{equation}\label{Group cost: group norm}
c_m^{\text{\rm Group}}\bigl(\underline{\mathbf{x}}^m,\underline{\mathbf{x}}^m+\underline{\boldsymbol{\delta}}^m\bigr) := \alpha \Vert \underline{\boldsymbol{\delta}}^m \Vert_{1,2} + (1-\alpha) \Vert \underline{\boldsymbol{\delta}}^m \Vert^2_F, 
\end{equation}
where $\underline{\boldsymbol{\delta}}^m = [\boldsymbol{\delta}_1;\ldots;\boldsymbol{\delta}_m ]$ is the perturbation matrix. Here the first term $\alpha \Vert \underline{\boldsymbol{\delta}}^m \Vert_{1,2}$ penalizes the perturbation matrix's group norm, while the second term $(1-\alpha)\Vert \underline{\boldsymbol{\delta}}^m \Vert^2_F = (1-\alpha)\sum_{i=1}^n \Vert \boldsymbol{\delta}_i \Vert^2 $ penalizes the magnitude of each perturbation. While $\alpha = 1$ only penalizes the group-norm term, $\alpha = 0$ leads to the standard optimal transport problem. 

\subsection{Nuclear Norm Cost: Low-rank Perturbations}
We propose a group cost function imposing a low-rank structure in the perturbation matrix. Note that a rank-$k$ space is the smallest linear subspace spanning $k$ linearly-independent vectors, and therefore the space of rank-$k$ perturbations gives a convex relaxation of a perturbation set with size $k$. As discussed in the introduction, universal perturbation sets can model batch effects in biological datasets \citep{luo2010comparison}.

In the optimization literature, the nuclear norm \citep{fazel2002matrix} provides a convex relaxation of a matrix's rank and is  applied to several signal recovery and matrix completion problems \citep{candes2009exact,recht2010guaranteed}. To model low-rank perturbations, we therefore define the following group cost function:
\begin{equation}\label{Group cost: Nuclear norm}
c_m^{\text{\rm Nuc}}\bigl(\underline{\mathbf{x}}^m,\underline{\mathbf{x}}^m+\underline{\boldsymbol{\delta}}^m\bigr) := \alpha \Vert \underline{\boldsymbol{\delta}}^m \Vert_{*} + (1-\alpha) \Vert \underline{\boldsymbol{\delta}}^m \Vert^2_F.  
\end{equation}
Since the Frobenius norm term reduces to the sinular values' Euclidean norm, the defined group cost in fact simplifies to an elastic net penalty \citep{zou2005regularization} on the perturbation matrix's singular values. As a result, the singular value decomposition (SVD) algorithm can be applied to compute this group cost function.

\section{Group-Structured Adversarial Training}
\subsection{Group-structured Distributionally Robust Optimization}
Using group-structured optimal transport costs, we develop a distributionally robust learning framework against group-structured perturbations. The following problem which we call \emph{Group-Structured Adversarial Training (GSAT)} represents the group-structured adversarial learning problem for group cost $c_m$ with size $m$:
\begin{equation}\label{Group Distributionally robust optimization}
    \min_{\mathbf{w}\in\mathcal{W}}\: \max_Q\: \mathbb{E}_Q\bigl[\, \ell(f_{\mathbf{w}}(\mathbf{X}),Y)\, \bigr] - \lambda W_{c_m}(Q,\hat{P}_n).
\end{equation}
Note that the group size $m$ and training size $n$ are two different parameters. In practice, we choose $m$ sufficiently large so that the desired group structure generalizes from the group of $m$ samples to the entire samples. 

The following theorem reduces the worst-case group risk function with permutation-invariant group costs into a risk minimization problem. Note that the reduction holds only for permutation-invariant group cost functions.  

\begin{thm}\label{Thm: group distributionally robust}
If $c_m$ is a non-negative lower semi-continuous permutation-invariant group cost, then
\begin{align}
 \max_Q\: \mathbb{E}_Q\bigl[\, \ell(f(\mathbf{X}),Y)\, \bigr] - \lambda W_{c_m}(Q,P)\, =\, \mathbb{E}_{P^m}\bigl[ \ell\circ f^{\lambda c_m} (\underline{\mathbf{X}}^m,\underline{Y}^m) \bigr] \nonumber  
\end{align}
with $\ell\circ f^{\lambda c}$ being the $c$-transform of the composition of $\ell$ and $f$ defined as \begin{align*}\ell\circ f^{\lambda c_m} (\underline{\mathbf{x}}^m,\underline{y}^m) := \max_{\underline{\mathbf{x}}'^m,\underline{y}'^m} \frac{1}{m}\sum_{i=1}^m\bigl[\ell(f(\mathbf{x}'_i),y'_i)  - \frac{\lambda}{m}c_m\bigl( (\underline{\mathbf{x}}^m,\underline{y}^m),(\underline{\mathbf{x}}'^m,\underline{y}'^m) \bigr)\bigr].
\end{align*}
\end{thm}
\begin{proof}
We defer the proof to the Appendix.
\end{proof}
This theorem shows that the GSAT problem \eqref{Group Distributionally robust optimization} for permutation-invariant cost functions, which applies to the three examples discussed in the previous section, reduces to the following risk minimization problem:
\begin{equation}\label{reduced eq: group distributionally robust}
    \min_{\mathbf{w}\in \mathcal{W}}\,    \mathbb{E}_{{\hat{P}_n}^m}\bigl[ \ell\circ f_{\mathbf{w}}^{\lambda c_m} (\underline{\mathbf{X}}^m,\underline{Y}^m) \bigr]. 
\end{equation}
Due to the definition of $c_m$-transform mapping, the above problem shows a minimax optimization task with a non-convex concave structure. Here, the concavity follows from a norm-based group cost function with sufficiently large strongly-convexity degree that is controlled by the product $\lambda(1-\alpha)$ in the three examples discussed. Also, the optimization objective represents the expectation of the $c_m$-transform with the $m$ group samples uniformly chosen from the $n$ training samples. We can use Monte-Carlo to approximate the gradient of this objective function. In our numerical experiments, we applied stochastic batch gradient descent with batch size $m$. Also, we set the transportation cost between two groups of labeled samples $\underline{(\mathbf{x},y)}^m$ and $\underline{(\mathbf{x}',y')}^m$ to be $+\infty$ if the label vectors $\underline{y}^m \neq\underline{y'}^m$ are different. Therefore, the above optimization problem further simplifies to
\begin{equation}\label{reduced eq: group distributionally robust_further reduced}
    \min_{\mathbf{w}\in \mathcal{W}}\,    \mathbb{E}_{{\hat{P}_n}^m}\bigl[ \ell( f_{\mathbf{w}} (\underline{\mathbf{X}}^{c_m}),\underline{Y}^m\bigr) \bigr],
\end{equation}
where $\underline{\mathbf{X}}^{c_m}$ denotes the solution to the group optimal transport problem $\max_{\underline{\mathbf{x}}'^m} \frac{1}{m}\sum_{i=1}^m\ell(f(\mathbf{x}'_i),y_i) - \frac{1}{m}c_m\bigl( (\underline{\mathbf{x}}^m,\underline{y}^m),(\underline{\mathbf{x}}'^m,\underline{y}^m) \bigr)$. Here we consider the same label vectors for both the original and transported samples. 

\subsection{GDADMM for GSAT Minimax Optimization}
For solving the GSAT problem \eqref{reduced eq: group distributionally robust}, we develop a batch gradient descent ascent (GDA) method with batch size equal to the group size $m$. This minimax optimization algorithm combining GDA with the widely-used alternating directions method of multipliers (ADMM) is specifically designed to handle non-smooth group cost functions which is the case in all the three examples discussed in previous sections.
 Algorithm~\ref{alg:GRM} contains the main steps of GDADMM for solving the GSAT minimax problem with stepsize parameters $\eta_0,\eta_1$. 

\begin{algorithm}[h]
   \caption{GDADMM for Solving GSAT}
   \label{alg:GRM}
\begin{algorithmic}
  \STATE Initialize parameters $\mathbf{w}^{(0)}$.
  \vspace{1mm}
  \FOR{$t=1$ {\bfseries to} $T_0$}
   \vspace{1mm}
   \STATE  Uniformly sample the group $(\mathbf{x}_i,y_i)_{i=1}^m$ from the training examples.
   \vspace{1mm}
   \STATE  Run ADMM-based Algorithm~\ref{alg:ADMM} to find the maximizer ${\underline{\boldsymbol{\delta}}^{m}}$ to the $c_m$-transform: $\max_{\underline{\boldsymbol{\delta}}^m} \frac{1}{m}\sum_{i=1}^m\ell(f_{\mathbf{w}^{(t)}}(\mathbf{x}_i+\boldsymbol{\delta}_i),y_i) - \frac{\lambda}{m}c_m(\underline{\mathbf{x}}^m+\underline{\boldsymbol{\delta}}^m,\underline{\mathbf{x}}^m)$.
   \vspace{1mm}
    \STATE  $\mathbf{w}^{(t+1)} = \mathbf{w}^{(t)} - \frac{\eta_0}{m}\sum_{i=1}^m\nabla_{\mathbf{w}}\,\ell(f_{\mathbf{w}^{(t)}}(\mathbf{x}_i+{{\boldsymbol{\delta}}}_i),y_i)$.
   \vspace{1mm}
  \ENDFOR
\end{algorithmic}
\end{algorithm}

Since the group cost function is in general non-smooth, one of the main optimization tasks for GSAT is to solve the following $c$-transform maximization problem for computing the perturbations of $m$ group samples:
\begin{align*}
\max_{\underline{\boldsymbol{\delta}}^m}\, \frac{1}{m}\sum_{i=1}^m\ell(f_{\mathbf{w}}(\mathbf{x}_i+{\boldsymbol{\delta}}_i),y_i) - \frac{1}{m} c_m( \underline{\mathbf{x}}^m,\underline{\mathbf{x}}^m + \underline{\boldsymbol{\delta}}^m ). 
\end{align*}
Note that each of the group-structured cost functions in  the previous section can be decomposed into a smooth Frobenius norm-squared term and a non-smooth convex norm term. The non-smooth term, which for simplicity we denote by $g$ in our analysis, imposes the desired group structure. For the three cases discussed in the paper, we have $ g^{\text{\rm Ind}}(\underline{\boldsymbol{\delta}}^m) = \sum_{i=1}^m \mathbf{1}(\boldsymbol{\delta}_i \neq \overline{\boldsymbol{\delta}})$, $g^{\text{\rm Group}}(\underline{\boldsymbol{\delta}}^m) = \Vert \underline{\boldsymbol{\delta}}^m\Vert_{1,2}$, and $g^{\text{\rm Nuc}}(\underline{\boldsymbol{\delta}}^m) = \Vert \underline{\boldsymbol{\delta}}^m\Vert_{*}$.
Therefore, we can simplify the above maximization problem  to
\begin{align}\label{ADMM: objective}
    &\max_{\underline{\boldsymbol{\delta}}^m}\,  \frac{1}{m}\sum_{i=1}^m\bigl[ \,  \ell(f_{\mathbf{w}}(\mathbf{x}_i+{\boldsymbol{\delta}}_i),y_i)  -\lambda (1-\alpha) \Vert \boldsymbol{\delta}_i \Vert_2^2 \, \bigr]\nonumber   - \frac{\lambda\alpha}{m} g( \underline{\boldsymbol{\delta}}^m ) \nonumber \\
   = &\max_{\tiny \begin{aligned}
    &\underline{\boldsymbol{\delta}}'^m , \underline{\boldsymbol{\delta}}^m: \\
    &\underline{\boldsymbol{\delta}}'^m =  \underline{\boldsymbol{\delta}}^m
    \end{aligned}} \frac{1}{m}\sum_{i=1}^m\bigl[\, \ell(f_{\mathbf{w}}(\mathbf{x}_i+{\boldsymbol{\delta}}_i),y_i) - \lambda(1-\alpha) \Vert \boldsymbol{\delta}_i \Vert_2^2 \, \bigr] - \frac{\lambda\alpha}{m} g( \underline{\boldsymbol{\delta}}'^m ). 
\end{align}

Given smooth loss $\ell$ and prediction function $f_{\mathbf{w}}$, the objective in \eqref{ADMM: objective} can be decomposed into a strongly-concave regularized loss function and a non-smooth concave negative norm function.
To optimize the sum of the smooth and non-smooth functions, we propose Algorithm~\ref{alg:ADMM} as a modified version of the ADMM algorithm \citep{boyd2004convex}. As detailed in Algorithm \ref{alg:ADMM}, we iteratively update the ADMM variables including matrix $\underline{\boldsymbol{\delta}}^m$ for the smooth term, matrix $\underline{\boldsymbol{\delta}}'^m$ for the non-smooth term, and Lagrangian matrix $\underline{\boldsymbol{\gamma}}^m$ for the linear constraint $\underline{\boldsymbol{\delta}}'^m = \underline{\boldsymbol{\delta}}^m$.

\begin{algorithm}[t]
   \caption{ADMM Step of GDADMM}
   \label{alg:ADMM}
\begin{algorithmic}
  \STATE Initialize $\forall i:\, \boldsymbol{\delta}_i^{(0)} =  {\boldsymbol{\delta}'}_i^{(0)}=\boldsymbol{\gamma}^{(0)}_i = \mathbf{0}$.
  \vspace*{2mm}
   \FOR{$t=1$ {\bfseries to} $T_1$}
   \vspace*{2mm}
   \STATE $\small\forall i:\, \boldsymbol{\delta}_i^{(t)} = (1-\frac{2\lambda(1-\alpha)\eta_1}{m})\boldsymbol{\delta}_i^{(t-1)} - \frac{\rho\eta_1}{m}\bigr( \boldsymbol{\delta}_i^{(t-1)} -  {\boldsymbol{\delta}'}_i^{(t-1)} - {\boldsymbol{\gamma}}_i^{(t-1)} \bigr) + \frac{\eta_1}{m}\nabla_{\boldsymbol{\delta}} \ell(f_{\mathbf{w}}(\mathbf{x}_i+\boldsymbol{\delta}_i^{(t-1)}),y_j) $,
   \vspace*{2mm}
   \STATE  $\small \underline{{\boldsymbol{\delta}}'}^{(t)} = \underset{\underline{\boldsymbol{\delta}}'}{\arg\!\min}\; \frac{\lambda\alpha}{m}g( \underline{\boldsymbol{\delta}}' ) + \frac{\rho}{2} \bigl\Vert  {\underline{\boldsymbol{\delta}}}^{(t)}- \underline{\boldsymbol{\delta}}'-  \underline{\boldsymbol{\gamma}}^{(t-1)} \bigr\Vert_F^2 $,
   \vspace*{2mm}
   \STATE $\small \forall i:\, {\boldsymbol{\gamma}}_i^{(t)} = {\boldsymbol{\gamma}}_{i-1}^{(t)} +\eta_0\bigl( {{\boldsymbol{\delta}}}_i^{(t)}- {{\boldsymbol{\delta}}_i'}^{(t)} \bigr) $
   \vspace*{2mm}
   \ENDFOR
\end{algorithmic}
\end{algorithm}

Note that unlike the standard ADMM algorithm, Algorithm~\ref{alg:ADMM} does not fully optimize the perturbation vector $\boldsymbol{\delta}_i$'s at every iteration and only apply one gradient ascent update similar to the ascent step of standard GDA algorithm. For the minimization step of Algorithm~\ref{alg:ADMM}, we use the following closed-form solutions for the cost functions discussed in the previous section:  
\begin{enumerate}[leftmargin=*,itemsep=0.3mm,topsep=0pt]
    \item For the indicator cost \eqref{Group cost: Indicator}, the optimal minimizer is: $\forall i,\: \underline{\boldsymbol{\delta}}'_i = \frac{1}{m}\sum_{j=1}^m \boldsymbol{\delta}^{(t)}_j  - \frac{1}{\rho} \boldsymbol{\gamma}^{(t-1)}_j$ .
    \item For the group-norm cost \eqref{Group cost: group norm}, the optimal minimizer is: $ \underline{\boldsymbol{\delta}}' = \Pi^{\text{\rm Group}}_{\frac{\lambda\alpha}{\rho m}} \bigl( \underline{\boldsymbol{\delta}}^{(t)}  - \frac{1}{\rho} \underline{\boldsymbol{\gamma}}^{(t-1)} \bigr) $, where $\Pi^{\text{\rm Group}}_{\xi}(A)$ multiplies each $A$'s column $A_i$ to $\max\{\Vert A_i\Vert_2 - \xi , 0\}/\Vert A_i\Vert_2 $.
    \item For the nuclear-norm cost \eqref{Group cost: Nuclear norm}, the optmal minimizer is: $ \underline{\boldsymbol{\delta}}' = \Pi^{\text{\rm Nuc}}_{\frac{\lambda\alpha}{\rho m}} \bigl( \underline{\boldsymbol{\delta}}^{(t)}  - \frac{1}{\rho} \underline{\boldsymbol{\gamma}}^{(t-1)} \bigr) $, where $\Pi^{\text{\rm Nuc}}_{\xi}(A)$ shrinks each $A$'s singular value $\sigma_i$ to $\max\{\sigma_i-\xi , 0 \}$.
\end{enumerate}
The following theorem shows that  the GDADMM algorithm is guaranteed to converge to a stationary minimax solution.
\begin{thm}\label{Thm: ADMM guarantee}
Suppose that for every $y\in \mathcal{Y}$ the loss function $\ell(f_{\mathbf{w}}(\mathbf{x}),y)$ is $L$-Lipschitz and $\beta$-smooth, i.e.
\begin{align*}
  \forall \mathbf{x},\mathbf{x}',\mathbf{w},\mathbf{w}':\;\; &\ell(f_{\mathbf{w}}(\mathbf{x}),y) - \ell(f_{\mathbf{w}'}(\mathbf{x}'),y)\le L\sqrt{\Vert \mathbf{x}-\mathbf{x}'\Vert^2+\Vert \mathbf{w}-\mathbf{w}'\Vert^2},\\
  &\Vert\nabla\ell(f_{\mathbf{w}}(\mathbf{x}),y) - \nabla\ell(f_{\mathbf{w}'}(\mathbf{x}'),y)\Vert \le \beta\sqrt{\Vert \mathbf{x}-\mathbf{x}'\Vert^2+\Vert \mathbf{w}-\mathbf{w}'\Vert^2}.
\end{align*}
Suppose that $ \beta \le \lambda(1-\alpha)$. Then, for the stepsize values $\eta_1=\mathcal{O}\bigl(\frac{1}{\beta}\bigr),\, \eta_0=\mathcal{O}\bigl(\frac{\lambda^2(1-\alpha)^2}{\bigl(\max\{1,1/\rho\}+\lambda(1-\alpha)\bigr)^2\beta}\bigr), $ 
Algorithm 1 finds an $\epsilon$-stationary minimax point where the optimal maximization value $F(\mathbf{w})$ satisfies $\Vert\nabla F(\mathbf{w})\Vert\le \epsilon$ in at most the following number of iterations:
\begin{equation}
    \mathcal{O}\biggl( \frac{(\beta+\rho)^2\bigl(1+\rho\lambda(1-\alpha)+\lambda^2(1-\alpha)^2\bigr)}{\epsilon^2}\biggr).
\end{equation}
\end{thm}
\begin{proof}
We defer the proof to the Appendix.
\end{proof}

\begin{figure}[t]
\vspace{- 0.0 cm}
\centering
\includegraphics[width=0.8\textwidth]{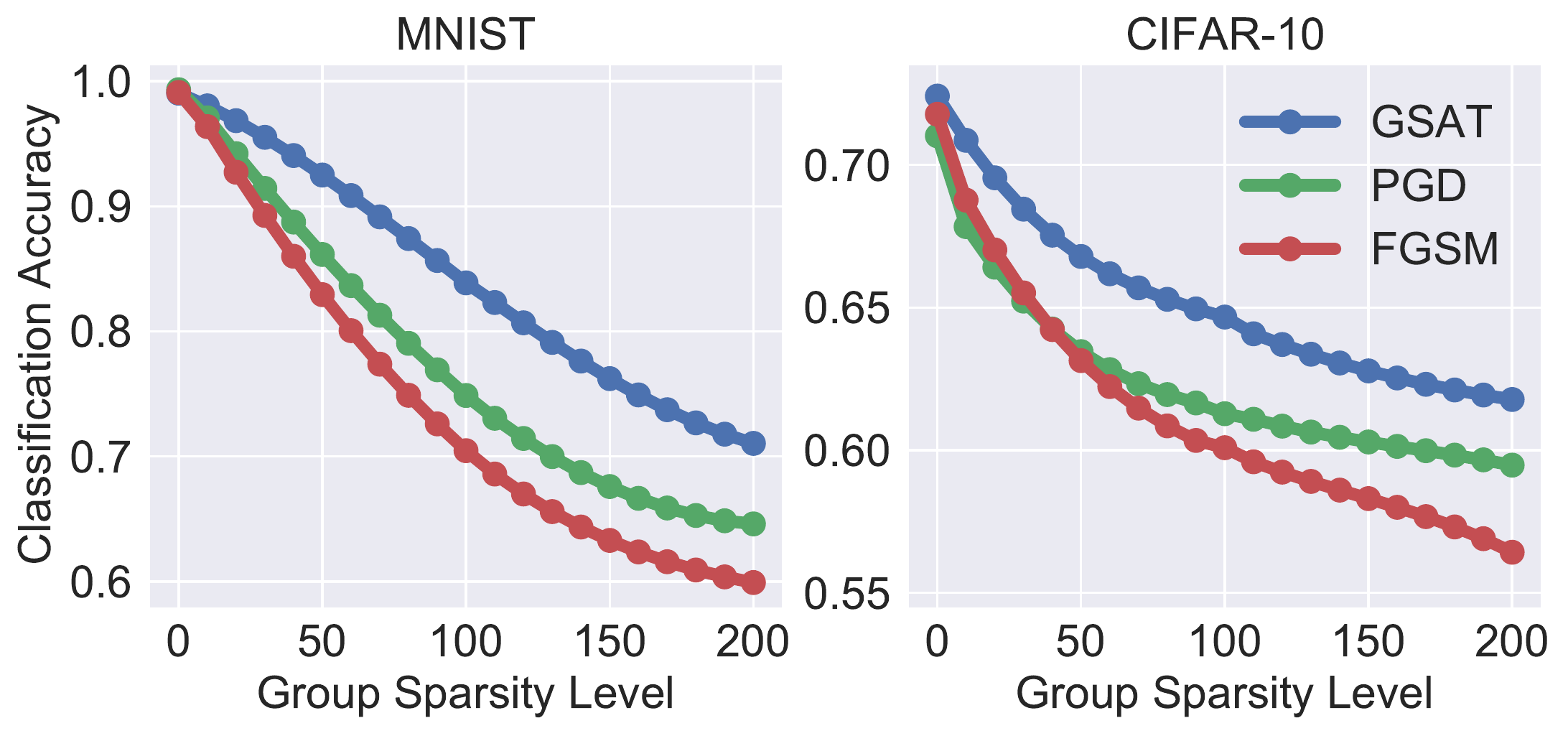}
\caption{GSAT compared to PGD and FGSM defense algorithms on MNIST and CIFAR-10 under group-sparse perturbations.} 
\label{fig:mnist_cifar_sparse}
\end{figure}

\section{Numerical Results}
In this section, we provide the numerical results of several applications of the proposed GSAT framework. In our experiments, we used the following datasets: MNIST \citep{lecun1998mnist} and CIFAR-10 \citep{krizhevsky2009learning} for image recognition, and HapMap GWAS dataset \citep{international2003international} and TCGA cancer atlas dataset \citep{tomczak2015cancer} for computational biology problems. We implemented the GSAT Algorithm~\ref{alg:GRM} in TensorFlow \citep{abadi2016tensorflow} and ran every experiment for $T_0=10^4$ iterations. In the experiments, batch-size and group size were chosen to be $m=200$. For the MNIST and CIFAR-10 experiments, 
we applied the AlexNet architecture \citep{krizhevsky2012imagenet}. For the experiments on GWAS and TCGA data, we applied a 1-hidden layer neural network with 100 smooth ELU \citep{clevert2015fast} neurons.

During test time, we designed the group-structured perturbations using the standard projected gradient descent (PGD) algorithm by properly projecting the perturbation vectors to control the universality, group-sparsity, and rank of the perturbation matrix after every gradient update. We applied 100 iterations of PGD updates with stepsize $0.001\mathbb{E}[\Vert \mathbf{X} \Vert_2]$ and considered a maximum $L_2$-norm of $0.05\, \mathbb{E}[\Vert \mathbf{X} \Vert_2]$ for the perturbation vectors. 

We chose stepsize parameters $\eta_0=10^{-4}$ and $\eta_1=10^{-1}$ in the experiments. In our implementation of Algorithm~\ref{alg:ADMM}, we used the following hyper-parameters: $\alpha = 0.5$, $\rho = 1$, and $T_1=20$. As a rule of thumb, we used $\lambda = 0.25\, \mathbb{E}[\Vert \mathbf{X}\Vert_2]$ to determine $\lambda$. We note that choosing large $\alpha$ and $\lambda$ values led to insufficient robustness to group-structured perturbations in our experiments, while choosing extremely small values for these hyper-parameters resulted in a significant drop in the standard test accuracy under no adversarial perturbations. Based on our experimental results, the proper value for $\lambda$ had an almost linear dependence on the average norm of input samples $\mathbb{E}[\Vert \mathbf{X}\Vert_2]$.

\subsection{GSAT Applied to Image Recognition Datasets}
For both MNIST and CIFAR-10 datasets, we performed four sets of experiments under universal, group-sparse, low-rank, and standard PGD adversarial perturbations. 
We applied GSAT, using indicator, group-norm, and nuclear-norm costs to defend against universal, group-sparse, low-rank, and standard unstructured attacks, respectively. We considered standard PGD defense \citep{madry2017towards} and Fast Gradient Sign Method (FGSM) \citep{goodfellow2014explaining} as the baselines, for which we considered a maximum perturbation norm equal to the average norm of the group-structured perturbations simulated by the GSAT algorithm. For the PGD baseline, we applied $20$ projected gradient steps with the stepsize $0.05\mathbb{E}[\Vert\mathbf{X}\Vert_2]$.

\begin{figure}[t]
\centering
\includegraphics[width=0.8\textwidth]{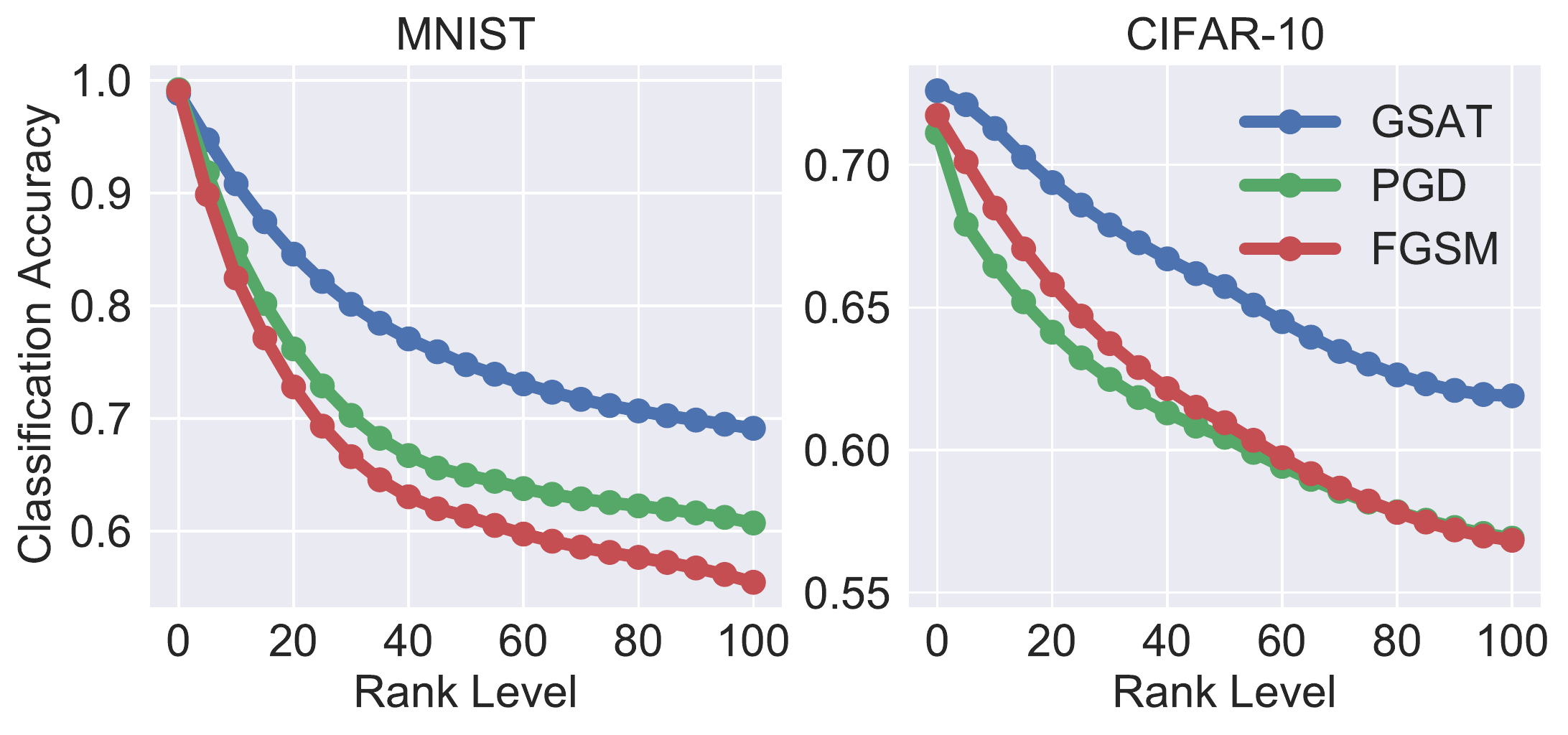}
\caption{GSAT compared to PGD and FGSM on MNIST and CIFAR-10 under low-rank perturbations.} 
\label{fig:mnist_cifar_rank}
\end{figure}

\begin{figure}[t]
\vspace{- 0.0 cm}
\centering
\includegraphics[width=0.8\textwidth]{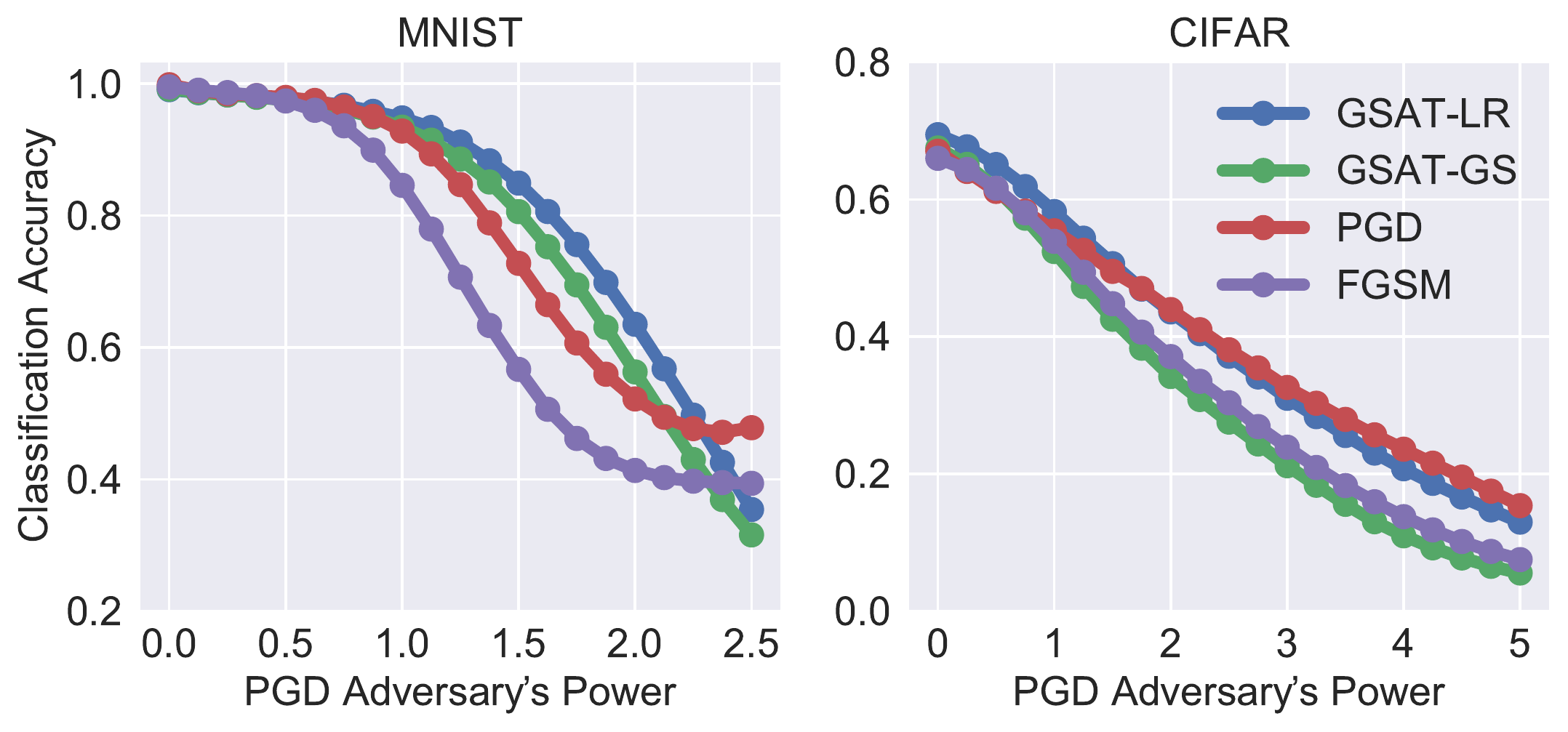}
\caption{Low-Rank GSAT (GSAT-LR) and Group-Sparse GSAT (GSAT-GS) compared to PGD and FGSM defense algorithms on MNIST and CIFAR-10 under standard PGD perturbations with different maximum perturbation norms.} 
\label{fig:mnist_cifar_pgd}
\end{figure}

For universal perturbations, the GSAT learner achieved an adversarial test accuracy of 97.6\% on MNIST and 75.2\% on CIFAR-10 under a universal perturbation, outperforming PGD and FGSM's learned models scoring 96.8\%, 96.4\% on MNIST and 74.5\%, 73.2\%  on CIFAR-10, respectively. For group-sparse perturbations, Fig.~\ref{fig:mnist_cifar_sparse} shows the adversarial test accuracy under group-sparse perturbations with sparsity parameter ranging from 0 to 200. In both MNIST and CIFAR-10 experiments, the proposed GSAT algorithm consistently outperformed standard PGD and FGSM training algorithms against group-sparse shifts. 
Fig.~\ref{fig:mnist_cifar_rank} replicates the experiments of Fig.~\ref{fig:mnist_cifar_sparse} for low-rank attacks with rank parameter changing between 0 and 100. As seen in this figure, the GSAT method outperformed the PGD and FGSM baselines on both MNIST and CIFAR-10 data. 

Finally, we observed that not only the GSAT framework improved robustness against group-sparse and low-rank perturbations, but also it resulted in a performance comparable to that of standard PGD and FGSM defense methods under standard PGD adversarial perturbations. Fig.~\ref{fig:mnist_cifar_pgd} shows the classification accuracy of the AlexNet models trained by rank-constrained GSAT (GSAT-LR) and group-sparse GSAT (GSAT-GS) against unstructured adversarial perturbations designed by the PGD algorithm. As seen in the figure, the achieved performance on both MNIST and CIFAR-10 test samples is similar to the performance of PGD and FGSM defense methods. We note that while for smaller PGD adversary's power values GSAT could even outperform standard PGD defense under standard PGD perturbations, for larger power values the PGD defense method performed marginally better.

In terms of computational speed, GSAT was only slightly slower than standard PGD training. In the MNIST experiments on one Tesla V100 GPU, GSAT with group-sparse and rank-constrained forms took on average $0.151$ and $0.139$ seconds per iteration, respectively, while every iteration of PGD training with the same number of inner maximization steps  took  $0.112$ seconds on average. Similarly, in the CIFAR-10 experiments on one Tesla V100 GPU, GSAT with group-sparsity and rank constraints needed $0.248$ and $0.344$ seconds per iteration, respectively, whereas PGD training spent $0.208$ seconds per iteration. Therefore, in the worst-case scenario GSAT was only $1.65$ times slower than standard PGD adversarial training, which suggests that GSAT's offered robustness against group-structured perturbations does not require significantly larger computational power than standard adversarial training.  

\subsection{GSAT Applied to Computational Biology Datasets}
\subsubsection{Robust Feature Selection in GWAS}
Genome-Wide Association Studies (GWAS) concern with identifying single nucleotide polymorphisms (SNPs) on the human DNA sequence that are associated with certain traits of an individual. Here, we apply the developed GSAT learning algorithm to perform classification and feature selection on a GWAS dataset containing 1,600 samples from the data collected in the international HapMap project \citep{international2003international}.
The features used include the first 3,475 SNPs on Chromosome 1. Here each SNP takes a value from $\{0,1,2\}$. The classification task is to identify the native continent of an individual from the following three categories: Africa (505 samples), Asia (675 samples), and Europe (420 samples). We randomly divided the data points into two sets of 1,200 training and 400 test samples.  We applied one-hot encoding to encode the SNP features, resulting in a total of 3$\times$3475$=$10425 zero-one features.

 We applied the proposed GSAT algorithm to learn a prediction rule robust against group-sparse perturbations. As shown in  Fig.~\ref{fig:gwas_tcga}, using the GSAT algorithm we achieved a considerably better performance compared to standard PGD and FGSM defense algorithms. Similar to the MNIST and CIFAR-10 experiments, we determined the PGD and FGSM's norm bound on perturbations to be the average power of the group-sparse perturbations in GSAT. 
 
 We then used the GSAT's trained classifier for feature selection and identified the top features targeted by the group-sparse adversary. This feature selection strategy based on the proposed GSAT framework applies to both linear and non-linear models. On the other hand, variable selection methods based on $L_1$-norm regularization such as the LASSO \citep{tibshirani1996regression} are limited to linear models and do not directly apply to non-linear DNNs. We defer further discussions on this group-sparse robust feature selection strategy to the Appendix. 
 
For the GSAT's trained ELU network, we selected the top 100 features targeted by the group-sparse structured perturbations. Using the selected 100 features, we trained another 1-hidden layer neural net with 100 ELU activation units and obtained $94.5\%$ test accuracy. This test accuracy was better than the $92.0\%$  accuracy obtained by the top $107$ features chosen by the LASSO \citep{tibshirani1996regression} and the $92.5\%$ accuracy for the 100 features chosen by a greedy gradient-based feature selection algorithm \citep{li2016deep}, selecting features by sorting the Euclidean norms of every gradient entry across training samples. 

\subsubsection{Learning Robust to Batch Effects in TCGA }
\vspace{-1mm}
The Cancer Genome Atlas (TCGA), a landmark cancer genomics program, has generated petabytes of genomics data over the past years which has led to improvements in diagnosis and treatment of cancer \citep{tomczak2015cancer}. In our experiments, we used a TCGA dataset publicly available on the UCI repository with 800 patients suffering from 5 different types of cancer as the label and 20,351 gene expression levels as the features. We divided the 800 samples to two sets of 500 training and 300 test samples.

\begin{figure}[t]
\vspace{- 0.0 cm}
\centering
\includegraphics[width=0.8\textwidth]{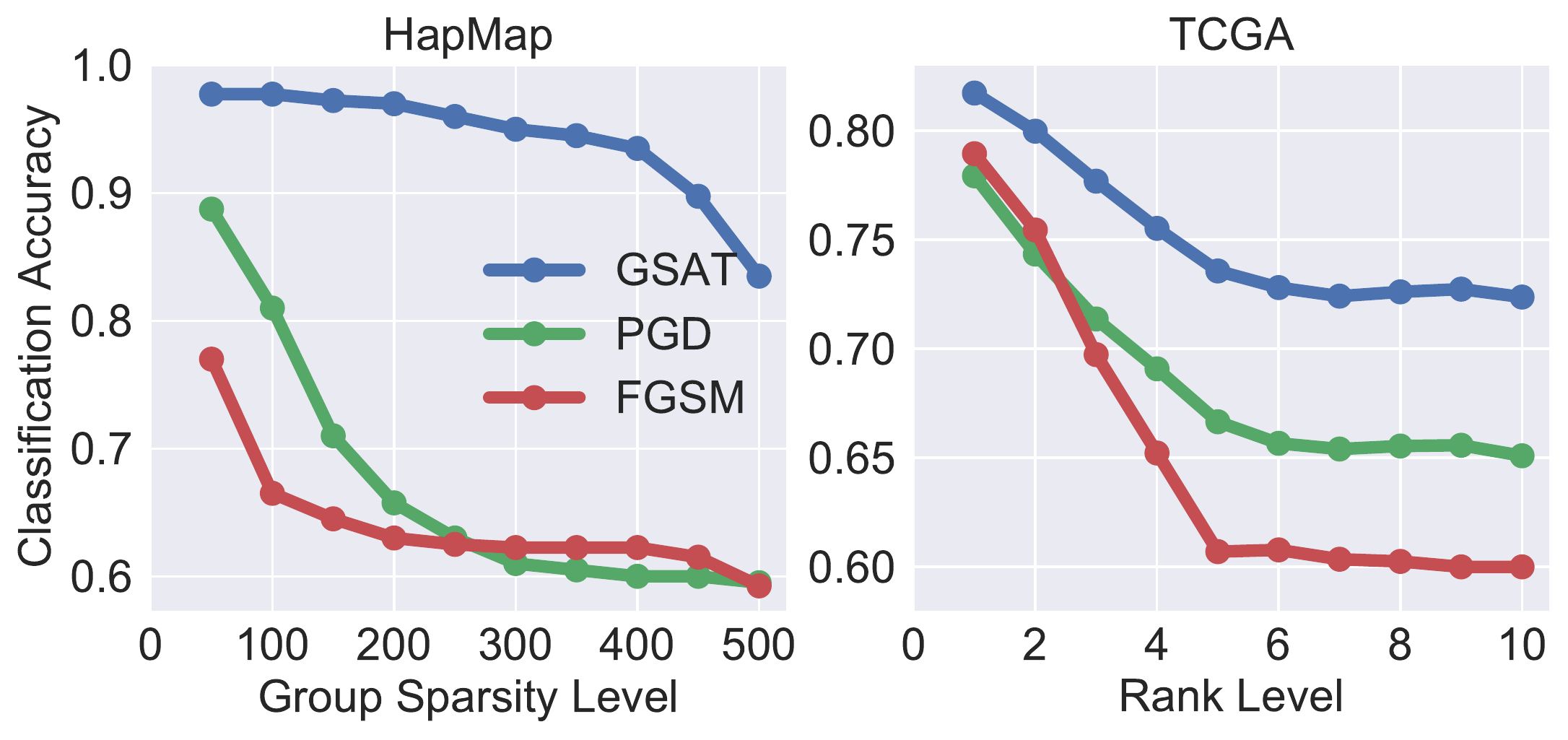}
\caption{ (Left) GSAT against PGD and FGSM on GWAS HapMap data under group-sparse attacks with different sparsity values. (Right) GSAT against PGD and FGSM on TCGA data under low-rank attacks with different rank values.} 
\label{fig:gwas_tcga}
\end{figure}

Since TCGA samples have been collected in multiple phases, TCGA datasets typically suffer from batch effects \citep{lauss2013monitoring}. In order to have a robust prediction under batch effects, we considered models learned by a GSAT learner robust to low-rank perturbations. Therefore, we applied the GSAT algorithm corresponding to a nuclear-norm group cost. While on the unperturbed test samples the GSAT's learned classifier performed as well as the ERM learner with $84.33\%$ accuracy, the GSAT's trained network improved robustness against low-rank perturbations. As shown in Fig.~\ref{fig:gwas_tcga}, GSAT achieved a better robustness performance against low-rank perturbations in comparison to standard PGD and FGSM baselines. As a result, the GSAT's trained classifier is expected to generalize better to other TCGA datasets under batch effects.

\section{Conclusion}

We developed the GSAT learning framework for training machine learning models robust against a general class of group-structured perturbations. Group-structured perturbations generalize the concept of universal adversarial attacks to perturbations sharing a common structure across samples. Under group-structured perturbations such as batch and cell-type effects, we showed that the models trained by the GSAT learning algorithm perform significantly better than standard adversarially-trained models. 
Our work opens up new avenues in studying and characterizing the uncertainties existing in real-world datasets through the foundations of optimal transport theory.


\bibliography{example_paper}
\appendix
\section{GSAT for Robust Feature and Basis Selection}\label{Section: Feature Selection}
Feature selection is a basic task to extract knowledge from high-dimensional datasets by selecting a subset of features which properly model the output label. In the learning literature, multiple feature selection algorithms have been proposed for linear models. A popular feature selection algorithm is the LASSO \citep{tibshirani1996regression}, which regularizes the $L_1$-norm of a linear model. However, such $L_1$-norm regularization methods do not direct apply to \emph{non-linear} models including deep neural nets. A heuristic used for feature selection on DNNs is based on perturbing and assigning a relevance score to features followed by selecting the features with the highest relevance scores \citep{li2016deep}. Such greedy  selection approaches are typically computationally expensive and become sub-optimal if the input features are highly correlated.

Here, we propose using the GSAT framework with the group-norm group cost for selecting the top relevant features influencing a DNN model's output the most. Here we identify the top $k$ features whose perturbation increases the loss function the most. To formulate an optimization problem for this idea, we solve the following problem for the feature subset $\Gamma$:
\begin{equation}\label{Feature Selection Equation}
    \underset{\Gamma:\, \operatorname{card}(\Gamma)\le k}{\arg\!\max}\: \min_{\mathbf{w}\in \mathcal{W}}\: \mathbb{E}\bigl[ \ell\bigl(\, f_\mathbf{w}\bigl(\mathbf{X} + \delta^{\text{\rm adv}(\Gamma)}(\mathbf{X})\bigr)\, , \,  Y\, \bigr) \bigr],
\end{equation}
where $\delta^{\text{\rm adv}(\Gamma)}(\mathbf{X})$ denotes the adversarial perturbation for input $\mathbf{X}$ applied to the features in $\Gamma$. The group-structured risk minimization problem generates group-sparse perturbations via the group-norm cost, and hence addresses this feature selection problem. As a result, the proposed group-structured robustness framework can be applied to give a robust feature selection algorithm.

We can extend this feature selection strategy to dimensionality reduction problems and similarly find the most relevant low-rank subspace in the space of features. To find the top $k$ relevant linear feature combinations, we update \eqref{Feature Selection Equation} to the following problem:
\begin{equation}\label{Basis Selection Equation}
    \underset{\Lambda:\, \operatorname{rank}(\Lambda)\le k}{\arg\!\max}\: \min_{\mathbf{w}\in \mathcal{W}}\: \mathbb{E}\bigl[ \ell\bigl(\, f_\mathbf{w}\bigl(\mathbf{X} + \delta^{\text{\rm adv}(\Lambda)}(\mathbf{X})\bigr)\, , \,  Y\, \bigr) \bigr].
\end{equation}
Here $\delta^{\text{\rm adv}(\Lambda)}(\mathbf{X})$ is a perturbation vector designed over the rank-bounded linear subspace $\Lambda$. Then, the group-structured risk minimization with the nuclear-norm cost  addresses this basis selection task under the condition that the simulated low-rank perturbations always target an identical low-rank subspace for every selection of $m$ samples.

\section{Proofs}
\subsection{Proof of Theorem 1}

According to our definitions,
\begin{align*}
    W_{c_m} ( P, Q)
    \stackrel{(a)}{=} &\inf_{\tiny \mathbf{M} \in \Pi( \underbrace{P\times\ldots \times P}_{\tiny m\, \text{\rm times}},\, \underbrace{ Q,\ldots,Q}_{\tiny m\, \text{\rm times}}) } \: \mathbb{E}_\mathbf{M} \bigl[ \frac{1}{m}c_m(\underline{\mathbf{X}}^m,{\underline{\mathbf{X}}'}^m)  \bigr] \\
    \stackrel{(b)}{=} & \inf_{\tiny \begin{aligned} & \mathbf{M} \in \Pi( P^m,\, \mathbf{R}_m) \\
    & \forall i , D_i: \mathbb{E}_{R_i}[D_i(\mathbf{X})]= \mathbb{E}_Q[D_i(\mathbf{X})] 
    \end{aligned} } \: \mathbb{E}_{\mathbf{M}} \bigl[ \frac{1}{m}c_m(\underline{\mathbf{X}}^m,{\underline{\mathbf{X}}'}^m)  \bigr] \\
    \stackrel{(c)}{=} & \max_{D_1 , \ldots , D_m} \: \inf_{\tiny \mathbf{M} \in \Pi( P^m, \mathbf{R}_m) } \: \mathbb{E}_\mathbf{M} \bigl[ \frac{1}{m}c_m(\underline{\mathbf{X}}^m,{\underline{\mathbf{X}}'}^m)  \bigr] \\
    & \quad + \frac{1}{m}\sum_{i=1}^m \bigl[ \, \mathbb{E}_Q[D_i(\mathbf{X})] - \mathbb{E}_{R_i}[D_i(\mathbf{X})] \, \bigr] \\
    \stackrel{(d)}{=} & \max_{D } \: \inf_{\tiny \mathbf{M} \in \Pi( P^m, \mathbf{R}_m) } \: \biggl\{ \mathbb{E}_\mathbf{M} \bigl[ \frac{1}{m}c_m(\underline{\mathbf{X}}^m,{\underline{\mathbf{X}}'}^m)  \bigr] \\
    & \quad + \frac{1}{m}\sum_{i=1}^m \bigl[ \, \mathbb{E}_Q[D(\mathbf{X})] - \mathbb{E}_{R_i}[D(\mathbf{X})] \, \bigr] \biggr\} \\
    \stackrel{(e)}{=} & \max_{D } \: \mathbb{E}_Q[D(\mathbf{X})] +
      \inf_{\mathbf{M} \in \tiny \Pi( P^m ,  \mathbf{R}_m) }  \biggl\{ \\ 
      & \quad  \mathbb{E}_{\mathbf{M}} \bigl[ \frac{1}{m}c_m(\underline{\mathbf{X}}^m,{\underline{\mathbf{X}}'}^m)  \bigr] - \frac{1}{m}\sum_{i=1}^m \mathbb{E}_{R_i}[D(\mathbf{X})] \biggr\} \\
    \stackrel{(f)}{=} & \max_{D } \: \mathbb{E}_Q[D(\mathbf{X})] + \inf_{\tiny \mathbf{M} \in  \Pi( P^m , \mathbf{R}_m) } \: \biggl\{ \\
    & \quad \mathbb{E} \bigl[ \frac{1}{m}c_m(\underline{\mathbf{X}}^m,{\underline{\mathbf{X}}'}^m)  \bigr]
     - \frac{1}{m}\sum_{i=1}^m \mathbb{E}_{R_i}[D(\mathbf{X})] \biggr\} \\
    \stackrel{(g)}{=} & \max_{D }\: \mathbb{E}_Q[D(\mathbf{X})] \\
    &\quad +  \mathbb{E}_\mathbf{P^m}  \bigl[ \inf_{{\underline{\mathbf{x}}'}^m}\:  \frac{1}{m} [ c_m(\underline{\mathbf{X}}^m,{\underline{\mathbf{x}}'}^m)  - \sum_{i=1}^m D(\mathbf{x}'_i) ] \bigr]\\
    \stackrel{(h)}{=} & \max_{D }\: \mathbb{E}_Q[D(\mathbf{X})] \\
    & \quad - \mathbb{E}_\mathbf{P^m}  \bigl[ \sup_{{\underline{\mathbf{x}}'}^m}\:  \frac{1}{m} \sum_{i=1}^m [ D(\mathbf{x}'_i) - c_m(\underline{\mathbf{X}}^m,{\underline{\mathbf{x}}'}^m) ] \bigr].\\
\end{align*}
Here, (a) follows from our definition of group optimal transport costs. (b) substitutes the marginal constraints with the equivalent $m$ constraints forcing each marginal distribution $R_i$ from the $m$-dimensional joint distribution optimization variable $\mathbf{R}_m$ to share all generalized moments for any function $D_i$ with that generalized moment of distribution $Q$. (c) uses the strong duality hold for the convex optimization problem and moves the constraints on the marginal distribution $R_i$'s to the objective. Note that the equivalent dual problem here is maximizing over possibly different $D_i$'s.

(e) utilizes the permutation invariance property of the group cost $c_m$. Note that the optimization problem over $D_i$'s is maximizing a concave objective which is also symmetric with respect to $D_i$'s, i.e. if we swap $D_i$ and $D_j$ for any two different indices $i\neq j$ the objective value will not change as a result of the symmetricity of the group cost function. Since we have a convex optimization problem with symmetric objective and constraints, there exists an optimal solution for $D^*_i$'s where $D^*_1=\cdots =D^*_m$. As a result, we can equivalently optimize only one function $D$ for a permutation invariant cost function. (f) holds because the term $\mathbb{E}_Q[D(\mathbf{X})]$ is independent from the optimization variables of the inner minimization problem. (g) holds since for a compact support set $\mathcal{X}$ we have the optimal joint-distribution $\mathbf{R}_m$ of random vector $\underline{\mathbf{X}'}^m$ as the distribution of a function of $\underline{\mathbf{X}}^m$ optimizing $\inf_{{\underline{\mathbf{x}}'}^m}\: c_m(\underline{\mathbf{X}}^m,{\underline{\mathbf{x}}'}^m)  - \sum_{i=1}^m D(\mathbf{x}'_i) $. Finally, (h) is a direct simplification which completes the proof.

\subsection{Proof of Theorem 2}

Without loss of generality we suppose $\lambda =1$, since the group cost $c_m$ can absorb any arbitrary $\lambda > 0$. Then,
\begin{align*}
& \sup_Q\: \mathbb{E}_Q\bigl[\, \ell(f(\mathbf{X}),Y)\, \bigr] -  W_{c_m}(Q,P) \\
 \stackrel{(a)}{=}\, & \sup_{Q}\: \sup_{\tiny \mathbf{M} \in \Pi( P^m,\, \underbrace{ Q,\ldots,Q}_{\tiny m\, \text{\rm times}}) } \biggl\{ \\
 &\mathbb{E}_Q \bigl[\, \ell(f(\mathbf{X}),Y)\, \bigr] - \mathbb{E}_\mathbf{M} \bigl[ \frac{1}{m}c_m(\underline{(\mathbf{X}',Y')}^m,{\underline{(\mathbf{X},Y)}}^m)  \bigr] \biggr\} \\
 \stackrel{(b)}{=}\, & \sup_{\scriptsize \begin{aligned}
&  Q, M \in \Pi( P^m,\, \mathbf{R}_m): \\
& \forall i:\: R_i=Q  
 \end{aligned} } \biggl\{ \\
 &\mathbb{E}_Q\bigl[\, \ell(f(\mathbf{X}),Y)\, \bigr] - \mathbb{E}_\mathbf{M} \bigl[ \frac{1}{m}c_m(\underline{(\mathbf{X}',Y')}^m,\underline{(\mathbf{X},Y)}^m)  \bigr] \biggr\} \\
 \stackrel{(c)}{=} \, & \sup_{\scriptsize \begin{aligned}
&  Q, M \in \Pi( P^m,\, \mathbf{R}_m): \\
& \forall i:\: R_i=Q  
 \end{aligned} } \biggl\{\\
 \mathbb{E}&\bigl[\,\frac{1}{m}\sum_{i=1}^m\bigl\{ \ell(f(\mathbf{X}_i),Y_i) -  c_m(\underline{(\mathbf{X}',Y')}^m,{\underline{(\mathbf{X},Y})}^m \bigr\} \bigr] \biggr\} \\
 \stackrel{(d)}{=} \, & \sup_{\scriptsize \begin{aligned}
&  \mathbf{R}_m, M \in \Pi( P^m,\, \mathbf{R}_m)
 \end{aligned} }\\
 \mathbb{E}&\bigl[\,\frac{1}{m}\sum_{i=1}^m\bigl\{ \ell(f(\mathbf{X}_i),Y_i) -  c_m(\underline{(\mathbf{X}',Y')}^m,({\underline{\mathbf{X},Y}})^m ) \bigr\} \bigr] \biggr\} \\
\stackrel{(e)}{=} \, &  \mathbb{E}_{P^m}\biggl[ \sup_{\underline{\mathbf{x}}'^m,\underline{y}'^m} \biggl\{  \\
&\quad \frac{1}{m}\sum_{i=1}^m\ell(f(\mathbf{x}'_i),y'_i) - c_m\bigl( (\underline{\mathbf{X},Y)}^m),(\underline{\mathbf{x}',y'})^m \bigr) \biggr\} \biggr] \\
\stackrel{(f)}{=} \, & \mathbb{E}_{P^m}\bigl[ \ell\circ f^{ c_m} (\underline{(\mathbf{X},Y)}^m ) \bigr]. 
\end{align*}
Here (a) follows from the definition of optimal transport costs. (b) merges the two maximization problems by defining a new optimization variable which is the joint distribution $\mathbf{R}_m$ with all its first-order marginals being $Q$. (c) follows from a simplification of the objective by taking all expectations over the joint distribution $M$.

(d) follows from the permutation invariance property of the group cost function $c_m$. Note that the optimization objective and constraints are all linear and symmetric, i.e. they do not alter if we swap the same two samples in groups ${\underline{(\mathbf{X},Y)}}^m$ and ${\underline{(\mathbf{X}',Y')}}^m$. Hence, there exists a solution with the same first-order marginals for $\mathbf{R}_m$. (e) is a result of the compactness of the support set $(\mathcal{X},\mathcal{Y})^m$ with the consequence that the optimal joint distribution $\underline{(\mathbf{X},Y)}^m$ follows from a function of random $\underline{(\mathbf{X}',Y')}^m$ solving $\sup_{\underline{\mathbf{x}}'^m,\underline{y}'^m}\: \frac{1}{m}\sum_{i=1}^m\ell(f(\mathbf{x}'_i),y'_i) - c_m\bigl( (\underline{\mathbf{X},Y)}^m),(\underline{\mathbf{x}',y'})^m \bigr) $. (f) immediately comes from the definition in the theorem and hence finishes the proof.

\subsection{Proof of Theorem 3}

We start by proving the following lemmas. Note that Lemma \ref{Lemma, Danskin theorem} is a direct consequence of the Danskin's theorem in \citep{bernhard1995theorem}. 
\begin{lemma}[Danskin's Theorem \citep{bernhard1995theorem}]\label{Lemma, Danskin theorem}
Suppose that $f(\mathbf{w},\boldsymbol{\delta})$ is a continuous function of $\mathbf{w},\boldsymbol{\delta}$ and is continuously differentiable with respect to $\mathbf{w}$. Define
\begin{equation}
    F(\mathbf{w}):= \max_{\boldsymbol{\delta}\in \Delta}\: f(\mathbf{w},\boldsymbol{\delta}).
\end{equation}
Then, if for every $\mathbf{w}$, $f(\mathbf{w},\boldsymbol{\delta})$ has a unique maximizer $\boldsymbol{\delta}^*(\mathbf{w})$ over the compact feasible set $\Delta$, $F(\mathbf{w})$ will be differentiable and satisfies
\begin{equation}
    \nabla F(\mathbf{w}) = \frac{\partial f}{\partial \mathbf{w}} (\mathbf{w},\boldsymbol{\delta}^*(\mathbf{w})).
\end{equation}
\end{lemma}

\begin{lemma}\label{Lemma, theorem 3}
Consider function $f(\mathbf{w},\boldsymbol{\delta})$. Assume that $f$ is $\beta$-smooth in $(\mathbf{w},\boldsymbol{\delta})$ and for every $\mathbf{w}$ $f(\mathbf{w},\cdot)$ is $\mu$-strongly-concave in $\boldsymbol{\delta}$, i.e. it is concave and satisfies
\begin{align*}
    \forall \boldsymbol{\delta},\boldsymbol{\delta}':\: \Vert\nabla_{\delta} f(\mathbf{w},\boldsymbol{\delta}) - \nabla_{\delta} f(\mathbf{w},\boldsymbol{\delta}')\Vert  \ge \mu \Vert\boldsymbol{\delta}-\boldsymbol{\delta}'\Vert.
\end{align*}
Then, for any concave function $g$ the following function will be $\beta/\mu$-Lipschitz in $\mathbf{w}$ $$\delta^*(\mathbf{w}):=\underset{\boldsymbol{\delta}}{\arg\!\max}\:f(\mathbf{w},\boldsymbol{\delta})+g(\boldsymbol{\delta}).$$ 
Also the following function will be $(\beta+\beta^2/\mu)$-smooth in $\mathbf{w}$:
$$F(\mathbf{w}):=\max_{\boldsymbol{\delta}}\:f(\mathbf{w},\boldsymbol{\delta})+g(\boldsymbol{\delta}).$$ 
\end{lemma}
\begin{proof}
First of all, note that under the above assumptions for every $\mathbf{w}$, $f(\mathbf{w},\boldsymbol{\delta})+g(\boldsymbol{\delta})$ is a $\mu$-strongly concave function of $\boldsymbol{\delta}$ with a unique-maximizer $\boldsymbol{\delta}^*(\mathbf{w})$. Here we use the notation $h_{\mathbf{w}}(\boldsymbol{\delta}):=f(\mathbf{w},\boldsymbol{\delta})+g(\boldsymbol{\delta})$ for simplicity. Since $h_{\mathbf{w}_1}$ and $h_{\mathbf{w}_2}$ are concave functions, given their optimal maximizers we have
\begin{align*}
    (\boldsymbol{\delta}^*(\mathbf{w}_2)-\boldsymbol{\delta}^*(\mathbf{w}_1))^T\nabla h_{\mathbf{w}_1}(\boldsymbol{\delta}^*(\mathbf{w}_1)) &\le 0, \\
    (\boldsymbol{\delta}^*(\mathbf{w}_1)-\boldsymbol{\delta}^*(\mathbf{w}_2))^T\nabla h_{\mathbf{w}_2}(\boldsymbol{\delta}^*(\mathbf{w}_2)) &\le 0.
\end{align*}
Therefore, we have
\begin{equation*}
   (\boldsymbol{\delta}^*(\mathbf{w}_2)-\boldsymbol{\delta}^*(\mathbf{w}_1))^T(\nabla h_{\mathbf{w}_2}(\boldsymbol{\delta}^*(\mathbf{w}_2))-\nabla h_{\mathbf{w}_1}(\boldsymbol{\delta}^*(\mathbf{w}_1))) \ge 0.
\end{equation*}
Due to the strong-concavity of $h_{\mathbf{w}_1}$ we further have
\begin{align*}
    (\boldsymbol{\delta}^*(\mathbf{w}_2)-\boldsymbol{\delta}^*(\mathbf{w}_1))^T(\nabla h_{\mathbf{w}_2}(\boldsymbol{\delta}^*(\mathbf{w}_1))-\nabla h_{\mathbf{w}_2}(\boldsymbol{\delta}^*(\mathbf{w}_2))) \\
    -\mu\Vert \boldsymbol{\delta}^*(\mathbf{w}_1) - \boldsymbol{\delta}^*(\mathbf{w}_2)\Vert^2 \ge 0.
\end{align*}
Therefore, we will have
\begin{align*}
    &\mu\Vert \boldsymbol{\delta}^*(\mathbf{w}_1) - \boldsymbol{\delta}^*(\mathbf{w}_2)\Vert^2 \\
    \le \, & (\boldsymbol{\delta}^*(\mathbf{w}_2)-\boldsymbol{\delta}^*(\mathbf{w}_1))^T(\nabla h_{\mathbf{w}_2}(\boldsymbol{\delta}^*(\mathbf{w}_1))-\nabla h_{\mathbf{w}_2}(\boldsymbol{\delta}^*(\mathbf{w}_2))) \\
    \le \, & (\boldsymbol{\delta}^*(\mathbf{w}_2)-\boldsymbol{\delta}^*(\mathbf{w}_1))^T(\nabla h_{\mathbf{w}_2}(\boldsymbol{\delta}^*(\mathbf{w}_1))-\nabla h_{\mathbf{w}_1}(\boldsymbol{\delta}^*(\mathbf{w}_1))) \\
    \le \, & \beta\Vert \mathbf{w}_2- \mathbf{w}_1\Vert\Vert \boldsymbol{\delta}^*(\mathbf{w}_2)-\boldsymbol{\delta}^*(\mathbf{w}_1)\Vert 
\end{align*}
which shows that

\begin{equation*}
    \Vert \boldsymbol{\delta}^*(\mathbf{w}_1) - \boldsymbol{\delta}^*(\mathbf{w}_2)\Vert\le\frac{\beta}{\mu}\Vert \mathbf{w}_2 - \mathbf{w}_1\Vert 
\end{equation*}

and completes the proof of the lemma's first part. For the second part, note that the objective in $F$'s definition has a unique maximizer. Hence, we can apply Lemma \ref{Lemma, Danskin theorem} to obtain
\begin{align*}
    \nabla F({\mathbf{w}}) &= \nabla_{\mathbf{w}}\bigl\{ f({\mathbf{w}},\boldsymbol{\delta}^*(\mathbf{w})) + g(\boldsymbol{\delta}^*(\mathbf{w}))\bigr\} \\
    &= \nabla_{\mathbf{w}} f({\mathbf{w}},\boldsymbol{\delta}^*(\mathbf{w})).
\end{align*}
Note that the above holds because $g(\boldsymbol{\delta})$ depends only on $\boldsymbol{\delta}$. Therefore, for every ${\mathbf{w}}_1,{\mathbf{w}}_2$:
\begin{align*}
    &\Vert \nabla F({\mathbf{w}}_1) - \nabla F({\mathbf{w}}_2) \Vert \\
    =\, & \Vert \nabla_{\mathbf{w}}  f({\mathbf{w}_1},\boldsymbol{\delta}^*(\mathbf{w}_1)) - \nabla_{\mathbf{w}} f({\mathbf{w}_2},\boldsymbol{\delta}^*(\mathbf{w}_2)) \Vert \\
     \le\, & \beta\bigl(\Vert\mathbf{w}_1-\mathbf{w}_2\Vert + \Vert\boldsymbol{\delta}^*(\mathbf{w}_2)-\boldsymbol{\delta}^*(\mathbf{w}_1)\Vert\bigr) \\
     \le\, & \beta(1+\beta/\mu)\Vert\mathbf{w}_1-\mathbf{w}_2\Vert,
\end{align*}
which completes the proof.
\end{proof}
Note that based on the theorem's assumptions, the minimax objective will be strongly-convex in $\boldsymbol{\delta}$ with degree $2\lambda(1-\alpha)-\beta\ge \lambda(1-\alpha)$. Considering the ADMM's variables, we have the following minimax optimization problem that is solved by Algorithm 1:
\begin{align*}
&\min_{\mathbf{w}}\; \max_{\tiny \underline{\boldsymbol{\delta}}^m} \frac{1}{m}\sum_{i=1}^m\bigl[\, \ell(f_{\mathbf{w}}(\mathbf{x}_i+{\boldsymbol{\delta}}_i),y_i)- \lambda(1-\alpha) \Vert \boldsymbol{\delta}_i \Vert_2^2 \, \bigr] \\     
&\quad  - \frac{\lambda\alpha}{m} g( \underline{\boldsymbol{\delta}}^m )\, = \\ 
   &\min_{\mathbf{w}}\; \max_{\tiny \begin{aligned}
    &\underline{\boldsymbol{\delta}}'^m , \underline{\boldsymbol{\delta}}^m: \\
    &\underline{\boldsymbol{\delta}}'^m =  \underline{\boldsymbol{\delta}}^m
    \end{aligned}} \frac{1}{m}\sum_{i=1}^m\bigl[\, \ell(f_{\mathbf{w}}(\mathbf{x}_i+{\boldsymbol{\delta}}_i),y_i)- \lambda(1-\alpha) \Vert \boldsymbol{\delta}_i \Vert_2^2 \, \bigr]  \\
    &\quad  - \frac{\lambda\alpha}{m} g( \underline{\boldsymbol{\delta}}'^m )\, = \\ 
    &\min_{\mathbf{w},\Gamma}\; \max_{\tiny \underline{\boldsymbol{\delta}}'^m , \underline{\boldsymbol{\delta}}^m} \frac{1}{m}\sum_{i=1}^m\bigl[\, \ell(f_{\mathbf{w}}(\mathbf{x}_i+{\boldsymbol{\delta}}_i),y_i) - \lambda(1-\alpha) \Vert \boldsymbol{\delta}_i \Vert_2^2 \, \bigr] \\
    &\quad  - \frac{\lambda\alpha}{m} g( \underline{\boldsymbol{\delta}}'^m ) -\frac{\rho}{2m}\Vert \underline{\boldsymbol{\delta}}^m-\underline{\boldsymbol{\delta}}'^m\Vert^2 + \frac{1}{m}\langle \Gamma,\underline{\boldsymbol{\delta}}^m-\underline{\boldsymbol{\delta}}'^m\rangle\, = \\
    &\min_{\mathbf{w},\Gamma} \max_{\tiny  \underline{\boldsymbol{\delta}}^m} \biggl\{ \frac{1}{m}\sum_{i=1}^m\bigl[\, \ell(f_{\mathbf{w}}(\mathbf{x}_i+{\boldsymbol{\delta}}_i),y_i) - \lambda(1-\alpha) \Vert \boldsymbol{\delta}_i \Vert_2^2] \\
    & -\min_{\tiny \underline{\boldsymbol{\delta}}'^m }\bigl\{ \frac{\lambda\alpha}{m} g( \underline{\boldsymbol{\delta}}'^m ) +\frac{\rho}{2m}\Vert \underline{\boldsymbol{\delta}}^m-\underline{\boldsymbol{\delta}}'^m\Vert^2 - \langle \frac{\Gamma}{m},\underline{\boldsymbol{\delta}}^m-\underline{\boldsymbol{\delta}}'^m\rangle\bigr\} \biggr\}.
\end{align*}
Here $\langle \cdot ,\cdot \rangle$ denotes the standard trace inner product. The minimax optimization problem in the last line is in the class of non-convex strongly-concave minimax problems with strong-concavity degree $2\lambda(1-\alpha)-\beta\ge \lambda(1-\alpha)$. Furthermore, according to Lemma \ref{Lemma, theorem 3} the last term in the minimax objective based on minimizing over $\underline{\boldsymbol{\delta}}'^m$ will be smooth with degree $(\rho+1)(1+\frac{\rho+1}{\rho})\le 8\max\{\rho,1\}$.

Algorithm 1 combined with Algorithm 2 can be seen to apply a gradient descent ascent algorithm for solving the above minimax optimization problem. We denote the maximized objective in the above maximization problems using the following notations:

\begin{align}
    F(\mathbf{w}):&=  \max_{\tiny \begin{aligned}
    &\underline{\boldsymbol{\delta}}'^m , \underline{\boldsymbol{\delta}}^m: \\
    &\underline{\boldsymbol{\delta}}'^m =  \underline{\boldsymbol{\delta}}^m
    \end{aligned}} \frac{1}{m}\sum_{i=1}^m\bigl[\, \ell(f_{\mathbf{w}}(\mathbf{x}_i+{\boldsymbol{\delta}}_i),y_i)- \lambda(1-\alpha) \Vert \boldsymbol{\delta}_i \Vert_2^2 \, \bigr]   - \frac{\lambda\alpha}{m} g( \underline{\boldsymbol{\delta}}'^m ), \\
    \tilde{F}(\mathbf{w},\Gamma):&= \max_{\tiny \underline{\boldsymbol{\delta}}'^m , \underline{\boldsymbol{\delta}}^m} \frac{1}{m}\sum_{i=1}^m\bigl[\, \ell(f_{\mathbf{w}}(\mathbf{x}_i+{\boldsymbol{\delta}}_i),y_i) - \lambda(1-\alpha) \Vert \boldsymbol{\delta}_i \Vert_2^2 \, \bigr] - \nonumber \\
    &\quad \frac{\lambda\alpha}{m} g( \underline{\boldsymbol{\delta}}'^m ) -\frac{\rho}{2m}\Vert \underline{\boldsymbol{\delta}}^m-\underline{\boldsymbol{\delta}}'^m\Vert^2 + \frac{1}{m}\langle \Gamma,\underline{\boldsymbol{\delta}}^m-\underline{\boldsymbol{\delta}}'^m\rangle.
\end{align}
Then, based on Theorem C.1 in \citep{lin2019gradient} Algorithm 1 with stepsize choices in Theorem 3 is guaranteed to obey the following over $T$ iterations
\begin{align}
    \frac{1}{T}\sum_{i=1}^T\big\Vert\nabla \tilde{F}(\mathbf{w}^{(i)},\Gamma^{(i)})\big\Vert^2_2 \nonumber\le\,  &\mathcal{O}\bigl(\frac{\kappa(\beta+\rho+1/\rho)(\kappa+\beta+\rho+1/\rho))}{T}\bigr)  \nonumber \\
    \le\,  &\mathcal{O}\bigl(\frac{(\beta+\rho)^2}{T}\bigr) 
    \label{Eq: Thm3 iteration complexity}
\end{align}
Here $\kappa=O(\frac{\beta+2\lambda(1-\alpha)+1+1/\rho}{\lambda(1-\alpha)})=O(3+\frac{2\max\{1,1/\rho\}}{\lambda(1-\alpha)})=O(1)$ is the condition number of the minimax problem. Also, note that $\Vert \nabla_{\Gamma} \tilde{F}(\mathbf{w}^{(i)},\Gamma^{(i)}) \Vert\le \Vert \nabla \tilde{F}(\mathbf{w}^{(i)},\Gamma^{(i)}) \Vert$. According to Lemma \ref{Lemma, Danskin theorem}, given the optimal $\underline{\boldsymbol{\delta}}^m(\mathbf{w},\Gamma),\underline{\boldsymbol{\delta}}'^m(\mathbf{w},\Gamma)$ for minimization variables $\mathbf{w},\Gamma$ we have
\begin{equation}
   \big\Vert \nabla_{\Gamma} \tilde{F}(\mathbf{w},\Gamma) \big\Vert =  \big\Vert\underline{\boldsymbol{\delta}}^m(\mathbf{w},\Gamma) - \underline{\boldsymbol{\delta}}'^m(\mathbf{w},\Gamma) \big\Vert.
\end{equation}
Therefore, assuming that $\Vert \nabla \tilde{F}(\mathbf{w}^{(i)},\Gamma^{(i)}) \Vert \le \epsilon$ implies that $\Vert\underline{\boldsymbol{\delta}}^m(\mathbf{w}^{(i)},\Gamma^{(i)}) - \underline{\boldsymbol{\delta}}'^m(\mathbf{w}^{(i)},\Gamma^{(i)}) \Vert\le \epsilon$. As a result, choosing $\underline{\boldsymbol{\delta}}^m=\underline{\boldsymbol{\delta}}'^m=\underline{\boldsymbol{\delta}}'^m(\mathbf{w}^{(i)},\Gamma^{(i)})$ will reach a min-max objective value that is at most $\frac{1}{2}(\beta+2\lambda(1-\alpha)+\rho)\epsilon^2$ below the optimal maximum value of the original minimax objective given $\mathbf{w}=\mathbf{w}^{(i)}$. Since the objective is $\lambda(1-\alpha)$-concave in $\underline{\boldsymbol{\delta}}^m$, we have 
$$\big\Vert \underline{\boldsymbol{\delta}}'^m(\mathbf{w}^{(i)},\Gamma^{(i)}) - \underline{\boldsymbol{\delta}}^m(\mathbf{w}^{(i)}) \big\Vert^2_2 \le {\frac{\bigl(3\lambda(1-\alpha)+\rho\bigr)\epsilon^2}{\lambda(1-\alpha)}}$$
As a result, the Danskin's theorem implies that given that $\Vert \nabla\tilde{F}(\mathbf{w}^{(i)},\Gamma^{(i)}) \Vert \le \epsilon$ we have
\begin{equation} \label{Eq: Thm 3 proof Eq 2}
   \Vert \nabla F(\mathbf{w}^{(i)}) \Vert \le \epsilon\left(1+\lambda(1-\alpha)\sqrt{3+\frac{\rho}{\lambda(1-\alpha)}}\right).
\end{equation}
Therefore, according to \eqref{Eq: Thm3 iteration complexity} and \eqref{Eq: Thm 3 proof Eq 2} over the following number of iterations we will find a first-order stationary $\mathbf{w}$ for the worst-case objective $\Vert\nabla F(\mathbf{w}) \Vert\le \epsilon$:
\begin{equation}
    \mathcal{O}\bigl(\frac{(\beta+\rho)^2\bigl(1+\rho\lambda(1-\alpha)+\lambda^2(1-\alpha)^2\bigr)}{\epsilon^2}\bigr).
\end{equation}
The above result completes the proof.

\end{document}